\def\year{2021}\relax
\documentclass[letterpaper]{article}
\usepackage{aaai21} 
\nocopyright
\usepackage{times} 
\usepackage{helvet}
\usepackage{courier} 
\usepackage[hyphens]{url} 
\usepackage{graphicx}
\usepackage{natbib}
\usepackage{caption}
\usepackage{amsmath, amssymb, amsthm}
\usepackage{tikz}
\usetikzlibrary{arrows,automata, positioning, calc}
\usepackage[ruled,vlined,linesnumbered,boxed,noend]{algorithm2e}
\usepackage{cleveref}
\usepackage{soul}
\usepackage{siunitx}
\usepackage{booktabs}
\usepackage{dsfont}
\usepackage{makecell}
\usepackage{cite}

\DeclareMathOperator*{\argmax}{arg\,max}

\DeclareRobustCommand{\MTSCDQN}{\mbox{CDQN }}

\newtheorem{theorem}{Theorem}

\newtheorem{proposition}{Proposition}

\makeatletter
\newcommand{\printfnsymbol}[1]{%
  \textsuperscript{\@fnsymbol{#1}}%
}
\makeatother

\title{Deep Constrained Q-learning}
\author{
  Gabriel Kalweit\textsuperscript{\rm 1}\thanks{Equal Contribution.}, Maria Huegle\textsuperscript{\rm 1}\printfnsymbol{1}, Moritz Werling\textsuperscript{\rm 2}, Joschka Boedecker\textsuperscript{\rm 1,3}\\
}
\affiliations{
    \textsuperscript{\rm 1}Neurorobotics Lab, University of Freiburg, Germany\\
    \textsuperscript{\rm 2}BMWGroup, Germany\\
    \textsuperscript{\rm 3}BrainLinks-BrainTools, University of Freiburg, Germany\\
}

\begin{document}
\maketitle

\begin{abstract}
In many real world applications, reinforcement learning agents have to optimize multiple objectives while following certain rules or satisfying a list of constraints. Classical methods based on reward shaping, i.e. a weighted combination of different objectives in the reward signal, or Lagrangian methods, including constraints in the loss function, have no guarantees that the agent satisfies the constraints at all points in time and can lead to undesired behavior. When a discrete policy is extracted from an action-value function, safe actions can be ensured by restricting the action space at maximization, but can lead to sub-optimal solutions among feasible alternatives. In this work, we propose Constrained Q-learning, a novel off-policy reinforcement learning framework restricting the action space directly in the Q-update to learn the optimal Q-function for the induced constrained MDP and the corresponding safe policy. In addition to single-step constraints referring only to the next action, we introduce a formulation for approximate multi-step constraints under the current target policy based on truncated value-functions. We analyze the advantages of Constrained Q-learning in the tabular case and compare Constrained DQN to reward shaping and Lagrangian methods in the application of high-level decision making in autonomous driving, considering constraints for safety, keeping right and comfort. We train our agent in the open-source simulator SUMO and on the real HighD data set.
\end{abstract}

\section{Introduction}
Deep reinforcement learning algorithms have achieved state-of-the-art performance in many domains in recent years \cite{DBLP:journals/nature/MnihKSRVBGRFOPB15, DBLP:journals/nature/SilverHMGSDSAPL16, DBLP:conf/nips/WatterSBR15, DBLP:journals/jmlr/LevineFDA16}. The goal for a reinforcement learning (RL) agent is to maximize the expected accumulated reward which it collects while interacting with its environment. However, in contrast to commonly used simulated benchmarks like computer games \cite{DBLP:journals/corr/abs-1207-4708} or MuJoCo environments \cite{6386109}, in real-world applications the reward signal is not pre-defined and has to be hand-engineered. Formulating an immediate reward function such that the outcome of the training process is consistent with the goals of the task designer can be very hard though, especially in cases where different objectives have to be combined. Nonetheless, it is crucial for many safety-critical tasks such as autonomous driving, warehouse logistics or assistance in health, amongst others. One way to approach this problem is to use a weighted sum in the immediate reward function, commonly known as \textit{reward shaping} (RS), and apply classical RL algorithms such as DQN \cite{DBLP:journals/nature/MnihKSRVBGRFOPB15} directly without further modifications. In practice, finding the suitable coefficients for the different objectives requires prior knowledge about the task domain or hyperparameter optimization which can be very time consuming. Other, more sophisticated multi-objective approaches \cite{MannorGeometric, Pirotta2014MultiobjectiveRL, 6889732} use multiple reward signals and value-functions and try to find Pareto-optimal solutions, i.e. solutions that cannot be improved in at least one objective. Picking one of the Pareto-optimal solutions for execution is, however, non-trivial. Another common approach to ensure consistency with constraints in Q-learning \cite{Watkins92q-learning}, referred to as \textit{Safe Policy Extraction} (SPE) in the following, is to restrict the action space during policy extraction \cite{Mirchevska2018HighlevelDM, tactical_decision_making}, masking out all actions leading to constraint violations. As we show in this work, however, this approach can lead to non-optimal policies under the given set of constraints.
\begin{figure*}
    \centering
   \includegraphics[width=0.8\textwidth]{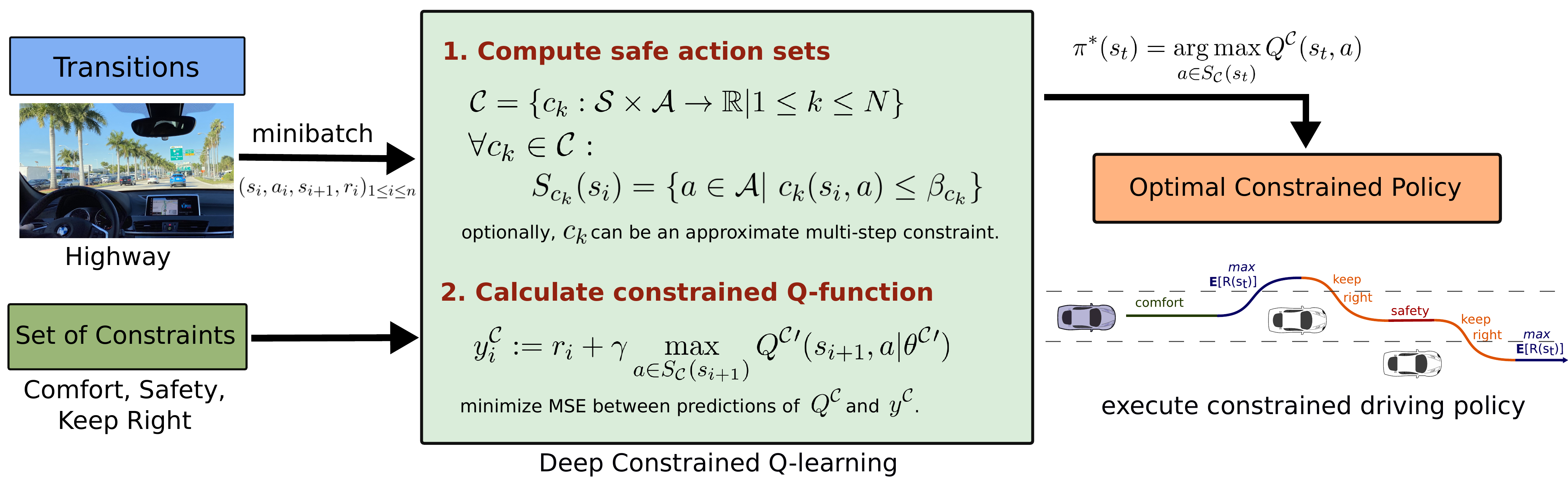}
    \caption{Scheme of Deep Constrained Q-learning for the application of autonomous lane changes. The agent receives a set of transitions and an accompanying set of constraints. For its update, the safe action set for the next state is retrieved in order to restrict the search space over value functions by action space shaping.}
    \label{fig:motivsumo}
\end{figure*}
Notably, in many applications there is one primary objective (e.g. driving as close as possible to a desired velocity) to be optimized, while additional auxiliary costs are used to guide the agent and ensure various side-constraints (e.g. avoid crashes or guarantee comfort). An exemplary setup with multiple objectives can be seen in \Cref{fig:motivsumo}. A common formulation for reinforcement learning with constraints is the constrained Markov Decision Process (CMDP) framework \cite{altman1999constrained}, where instead of a weighted combination of the different objectives, agents are optimizing one objective while  satisfying constraints on expectations of auxiliary costs. We propose a novel Q-learning algorithm that satisfies a list of single-step and multi-step constraints, where we model multi-step constraints as expectations of auxiliary costs as in the CMDP framework. These multi-step constraints, however, are estimated via truncated value-functions \cite{composite_q}, to approximate constraint costs over the next $H$ steps following the current target-policy. The benefit of this formulation is that constraints are independent of the scaling of the immediate reward function and can act on different time scales which allows for an easier and more general formulation of constraints. Further, this new class of constraint formulations does not require discounting in the multi-step case. This leads to a much more intuitive way to formulate proper constraints for the agent, e.g. constraints such as performing less than two lane changes in $\SI{10}{\second}$  or performing lane-changes only if the velocity gain in the next  $\SI{10}{\second}$ exceeds $\SI{0.25}{\metre\per\second}$. To formulate a corresponding constraint based on discounted value-functions, the discount value has to be set precisely to cover a time frame of $\SI{10}{\second}$, along with immediate penalties which in sum have to represent a total of two lane changes or an increase in velocity of $\SI{0.25}{\metre\per\second}$.\\

Our contributions are threefold. First, we define an extension of the update in Q-learning which modifies the action selection of the maximization step to ensure an optimal policy with constraint satisfaction in the long-term, an algorithm we call \textit{Constrained Q-learning}. We further show that the constrained update leads to the optimal deterministic policy for the case of Constrained Policy Iteration and show that Constrained Q-learning can lead to a drastic reduction of the search space in the tabular case compared to reward shaping. Second, we introduce a new class of multi-step constraints which refer to the current target policy. Third, we employ Constrained Q-learning within DQN and evaluate its performance in high-level decision making for autonomous driving. We show that Constrained DQN (CDQN) is able to outperform reward shaping, Safe Policy Extraction and Lagrangian optimization techniques and further use the open HighD data set \cite{highDdataset}, containing 147 hours of top-down recordings of German highways, to learn a smooth and anticipatory driving policy satisfying traffic rules to further highlight the real-world applicability of Constrained Q-learning.

\section{Related Work}
\begin{table}[t]
    \centering
    \caption{Overview of possible constraint types in $^{*}$\cite{DBLP:journals/corr/abs-1805-11074}, $^{**}$\cite{Dalal} and $^{***}$\cite{achiam2017constrained}. Comparison to CDQN (ours) and reward shaping (RS).}
    \begin{tabular}{ccccc}
         \toprule
         & Discounted&Mean&\thead{Reward\\Agnostic}&\thead{Separable\\Time Scale}\\
         \midrule
         RCPO$^{*}$ &$\checkmark$&$\checkmark$&$\checkmark$&$\times$\\
         SL$^{**}$&$\times$&$\times$&$\checkmark$&$\times$\\
         CPO$^{***}$&$\checkmark$&$\times$&$\checkmark$&$\times$\\
         \thead{RS}&$\checkmark$&$\checkmark$&$\times$&$\times$\\
         \midrule
         \textbf{CDQN}&$\checkmark$&$\checkmark$&$\checkmark$&$\checkmark$\\
         \bottomrule
    \end{tabular}
    \label{tab:comp}
\end{table}

A plethora of work exists to find solutions for CMDPs, most of them belonging to (1) Trust region methods \cite{achiam2017constrained} or (2) Lagrange multiplier methods \cite{Bertsekas/99, BorkarActorCritic, Bhatnagar2012}, where the CMDP is converted into an equivalent unconstrained problem by making infeasible solutions sub-optimal. However, these methods only guarantee near-constraint satisfaction at each iteration. In Reward Constraint Policy Optimization (RCPO), constraints are represented by reward penalties which are added to the immediate reward function via optimized Lagrange multipliers \cite{DBLP:journals/corr/abs-1805-11074}. Since the approach optimizes both long-term reward and long-term penalty simultaneously, no clear distinction between return and constraint violation can be formalized. This stands in contrast to our work, where return and constraints are decoupled and can act on different time-scales which enhances interpretability of the learned behaviour. Put differently, our approach provides the possibility to \textit{formulate} constraints w.r.t. a shorter horizon, but \textit{optimizes} satisfaction of these constraints on the long-term, as shown in \Cref{fig:mbmf} for the exemplary application of autonomous driving. Further, RCPO is an on-policy method, whereas our approach belongs to the family of off-policy Q-learning algorithms. Off-policy RL algorithms also have been combined with SPE in \cite{Mirchevska2018HighlevelDM, tactical_decision_making, Dalal}. However, as we show below, SPE is not guaranteed to yield the optimal action given the constrained MDP, which stands in contrast to our work. See \Cref{tab:comp} for an overview of the different approaches. Robust control methods are able to model constraints on the short-term horizon and ensure their long-term satisfaction through constraints on terminal costs \cite{torsten1, torsten2, torsten3}. However, they rely on accurate models of the environment. Our approach combines the intuitive formulation of constraints on the short-term horizon as in model-based approaches with the robustness of a model-free RL method for the long-term optimization.

\begin{figure}[h]
\centering
  \includegraphics[width=0.45\textwidth, trim={0 0.3cm 0 0}, clip]{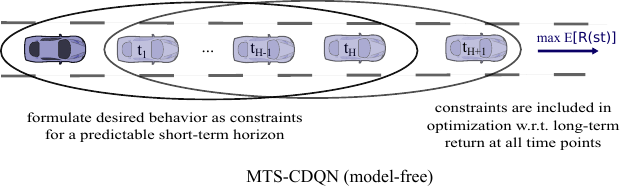} 
  \caption{Methodology of Constrained Q-learning, exemplary for the application of autonomous driving. Traffic rules are ensured in predictable short-term horizon. Long-term goals are optimized by optimization of long-term return.}
  \label{fig:mbmf}
\end{figure}
\section{Preliminaries}
In this section, we define the theoretical background.
\subsection{Markov Decision Processes (MDP)}
In a reinforcement learning setting, an agent interacts with an environment, which is typically modeled as an MDP $\langle\mathcal{S}, \mathcal{A}, \mathcal{P}, r, \gamma\rangle$. The agent is following policy $\pi : \mathcal{S} \rightarrow \mathcal{A}$ in some state $s_t$, applying a discrete action $a_t \sim \pi$ to reach a successor state $s_{t+1}\sim\mathcal{P}$ according to a transition model $\mathcal{P}$. In every discrete time step $t$, the agent receives reward $r_{t}$ for selecting action $a_t$  in state $s_t$. The goal of the agent is to maximize the expectation of the discounted long-term return $\mathbf{E}_{a_i\sim\pi, s_{i>t}\sim\mathcal{P}}[R(s_t)] = \mathbf{E}_{a_i\sim\pi, s_{i>t}\sim\mathcal{P}}[\sum_{i\geq t} \gamma^{i-t}r_i]$, where $\gamma \in [0, 1]$ is the discount factor. The action-value function $Q^\pi(s_t,a_t)=\mathbf{E}_{a_{i>t}\sim\pi, s_{i>t}\sim\mathcal{P}}[R(s_t)|a_t]$ represents the value of following a policy $\pi$  after applying action $a_t$. The optimal policy can be inferred from the optimal action-value function $Q^*(s_t, a_t) = \max_{\pi} Q^{\pi}(s, a)$ by maximization.

\subsection{Constrained Markov Decision Processes (CMDP)}
We consider a  CMDP $\langle\mathcal{S}, \mathcal{A}, \mathcal{P}, r, \gamma, \mathcal{C}\rangle$, with constraint set $\mathcal{C}$, where $ \mathcal{C}  = \{ c_k:\mathcal{S} \times \mathcal{A} \rightarrow \mathbb{R} | 1 \le k \le N \}$. We define the set of safe actions for a constraint $c_k \in \mathcal{C}$ as $S_{c_k}(s_t) = \{a \in \mathcal{A} | \  c_k(s_t, a) \le \beta_{c_k} \}$. We define $S_\mathcal{C}(s_t)$ as the intersection of all safe sets.

\subsection{Safe Policy Extraction}

Given an action-value function $Q$ and a set of constraints $\mathcal{C}$, we can extract the optimal safe policy $\pi$ w.r.t. $Q$ by $\pi(s_t) = \argmax_{a\in S_\mathcal{C}(s_t)}Q(s_t, a)$. We call this method \textit{Safe Policy Extraction}, abbreviated by \textit{SPE}.

\begin{figure}[t]
  \resizebox{0.45\textwidth}{!}{%
\begin{tikzpicture}[->,>=stealth',shorten >=1pt,auto,node distance=2.8cm,
                    semithick]
  \tikzstyle{every state}=[]

  \node[state] (0) {$s_0$};
  
  \node[state] (1) [right=1.5cm of 0, minimum width=1cm] {$s_1$};
  \node       (d1) [above=0.5cm of 1, minimum width=1cm] {};
  \node       (d2) [below=0.5cm of 1, minimum width=1cm] {};
  
  \node       (d3) [right=1.5cm of 1, minimum width=1cm] {};
  \node[state] (2) [right=1.5cm of d1, minimum width=1cm] {$s_2$};
  \node[state] (3) [right=1.5cm of d2, minimum width=1cm] {$s_3$};

  \node       (d4) [right=1.5cm of d3, minimum width=1cm] {};
  \node[state] (4) [right=1.5cm of 2, minimum width=1cm] {$s_4$};
  \node[state] (5) [right=1.5cm of 3, minimum width=1cm] {$s_5$};

  \node[state] (7) [right=1.5cm of d4, minimum width=1cm] {$s_7$};
  \node[state] (6) [right=1.5cm of 4, minimum width=1cm, dashed] {$s_6$};
  \node[state] (8) [right=1.5cm of 5, minimum width=1cm] {$s_8$};

  \node[state] (9) [right=1.5cm of 6, minimum width=1cm] {};
  \node[state] (9i) [right=1.5cm of 6, minimum width=0.75cm, xshift=0.0625cm] {$s_9$};
  \node[state] (10) [right=1.5cm of 7, minimum width=1cm] {};
  \node[state] (10i) [right=1.5cm of 7, minimum width=0.75cm, xshift=0.0625cm] {$s_{10}$};
  \node[state] (11) [right=1.5cm of 8, minimum width=1cm] {};
  \node[state] (11i) [right=1.5cm of 8, minimum width=0.75cm, xshift=0.0625cm] {$s_{11}$};

  \path (0) edge [below, very thick] node {$+0$} (1)
        (1) edge [very thick, sloped, anchor=center, below] node {$a,+0$} (2)
        (1) edge [very thick, sloped, anchor=center, below] node {$b,+0$} (3)
        (2) edge [very thick, below] node {$+0$} (4)
        (4) edge [very thick] node {$a, +0$} (6)
        (6) edge [very thick] node {$+3$} (9)
        (4) edge [very thick, sloped, anchor=center, below] node {$b,+0$} (7)
        (7) edge [very thick, below] node {$+1$} (10)
        (3) edge [very thick, below] node {$+0$} (5)
        (5) edge [very thick, below] node {$+0$} (8)
        (8) edge [very thick, below] node {$+2$} (11);
        
    \draw[red, dashed, very thick] ($(0)+(0,0.6)$) -- ($(1)+(0,0.6)$) -- ($(2)+(0,0.6)$) --
    ($(4)+(0,0.6)$) --
    ($(6)+(0,0.6)$) --
    ($(9)+(0,0.6)$);
    \draw[blue, dashed, very thick] ($(0)+(0,0.4)$) -- ($(1)+(0,0.4)$) -- ($(2)+(0,0.4)$) --
    ($(4)+(0,0.4)$) --
    ($(7)+(0,0.4)$) --
    ($(10)+(0,0.4)$);
    \draw[green, dashed, very thick] ($(0)+(0,0.2)$) -- ($(1)+(0,0.2)$) -- ($(3)+(0,0.2)$) --
    ($(5)+(0,0.2)$) --
    ($(8)+(0,0.2)$) --
    ($(11)+(0,0.2)$);
\end{tikzpicture}
} 
  \caption{ In this MDP, state $s_6$ is marked as unsafe. The initial state is $s_0$. Unconstrained Q-learning (red) chooses the unsafe path leading to $s_9$ with a return of $+3$ and Safe Policy Extraction (blue) after Q-learning leads to a safe path to state $s_{10}$ with a return of $+1$. Constrained Q-learning (green) chooses the safe path to $s_{11}$ with a return of $+2$.}
  \label{fig:motivation}
\end{figure}
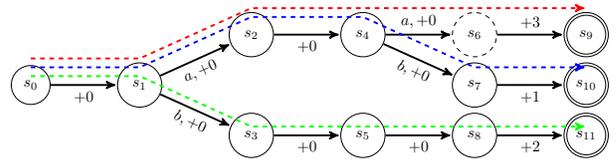

\begin{proposition}
Given an MDP $\mathcal{M}$ and set of constraints $\mathcal{C}$, SPE after Q-learning is not guaranteed to give the optimal safe policy for the induced constrained MDP $\mathcal{M_\mathcal{C}}$.
\end{proposition}
\begin{proof}
Follows from the counter example in \Cref{fig:motivation}.
\end{proof}

\section{Constrained Q-learning}
\label{sec:ltcq}
We extend the Q-learning update to use a set of constraints $\mathcal{C}$ with corresponding safe action set $S_{\mathcal{C}}(s_{t+1})$:
\begin{equation}
    \begin{split}
        Q^\mathcal{C}(s_t, a_t) \leftarrow&(1 - \alpha) Q^\mathcal{C}(s_t, a_t)\\&\,\,\,\,\,\,\,\,+ \alpha (r + \gamma \mathop{\max_{\ a \in S_{\mathcal{C}}(s_{t+1})}} Q^\mathcal{C}(s_{t+1}, a)),
    \end{split}
\end{equation}
with learning rate $\alpha$. The optimal deterministic constrained policy $\pi^*$ can then be extracted by:
$$\pi^*(s_t) = \mathop{\argmax_{a \in S_{\mathcal{C}}(s_t) }} Q^\mathcal{C}(s_{t}, a).$$

The effect of the constrained Q-update can be seen in \Cref{fig:motivation}. In the given MDP, state $s_6$ is marked as unsafe and has to be avoided. Vanilla Q-learning without knowledge about this constraint leads to a policy choosing the upper path to $s_9$ with a reward of $+3$ (in our experiments, we further discuss the case of RS). A safety check at policy extraction can then be used to avoid this unsafe path, however at the point of decision it can only choose the path leading to $s_{10}$ with a non-optimal return of $+1$. Incorporating the constraint in the Q-update directly propagates the non-optimal value of the upper path back to $s_1$, such that Constrained Q-learning converges to the optimal constrained policy leading to $s_{11}$ with a return of $+2$.

\section{Multi-step Constraints}
\label{sec:ltc}
While common constraints are only dependent on the current decision step, it can be crucial to represent the effect of the current policy of the agent for a longer time scale. Typically, long-term dependencies on constraints are represented by the expected sum of discounted or average constraint signals $j_i$, i.e. $\mathcal{J}^\pi(s_t, a_t) = \mathbf{E}_{a_{i>t}\sim\pi, s_{i>t}\sim\mathcal{P}}[J(s_t)|a_t] = \mathbf{E}_{a_{i>t}\sim\pi, s_{i>t}\sim\mathcal{P}}[\sum_{i\geq t} \gamma^{i-t}j_i]$.\\

Instead, in our approach we propose to only consider the next $H$ steps: $\mathcal{J}^\pi_H(s_t, a_t) = \mathbf{E}_{a_{i>t}\sim\pi, s_{i>t}\sim\mathcal{P}}[J_H(s_t)] = \mathbf{E}_{a_{i>t}\sim\pi, s_{i>t}\sim\mathcal{P}}[\sum_{i\geq t}^{t + H} j_i]$. Since the values are guaranteed to be bounded due to the fixed horizon $H$, discounting is not needed, which also leads to more interpretable constraints. We apply the formulation of truncated value-functions defined in \cite{composite_q} to predict the truncated constraint-values.\\

We first estimate the immediate constraint-value and then follow a consecutive bootstrapping scheme to get to the estimation of the full horizon $H$. The update rules for constraint-values $\mathcal{J}^\pi_h$ in the tabular case are:
\begin{equation}
    \begin{split}
    \mathcal{J}^\pi_{1}(s_t, a_t)\leftarrow&(1-\alpha_{\mathcal{J}}) \mathcal{J}^\pi_1(s_t, a_t)+\alpha_{\mathcal{J}} j_t \text{ and}\\
    \mathcal{J}^\pi_{h>1}(s_t, a_t)\leftarrow&(1-{\alpha_{\mathcal{J}}}) \mathcal{J}^\pi_h(s_t, a_t)+\alpha_{\mathcal{J}} (j_t+  \\
    &\mathcal{J}^\pi_{h-1}(s_{t+1}, \argmax_{a\in S_{\mathcal{C}}(s_{t+1})} Q(s_{t+1}, a))),
    \end{split}
\end{equation}
with constraint-specific learning rate $\alpha_{\mathcal{J}}$.

\begin{theorem}
Given an MDP $\mathcal{M}$ and set of multi-step constraints $\mathcal{C}$, Constrained Policy Iteration (CPI) converges to the optimal deterministic policy $\pi^*$ for the induced constrained MDP $\mathcal{M_\mathcal{C}}$.
\end{theorem}
\begin{proof}
Given a set of constraints $\mathcal{C}$ and the maximum horizon $H$ of all constraints, we can define the truncated constraint violation function $\mathcal{J}^{\pi_k}_H$ of horizon $H$ by
$\mathcal{J}^{\pi_k}_1(s)=\mathds{1}_{\pi_k(s)\not\in S_{\mathcal{C}}(s)}$ and
$\mathcal{J}^{\pi_k}_{h>1}(s)=\sum_{s^\prime}p(s^\prime|s,\pi_k(s))(\mathds{1}_{\pi_k(s)\not\in S_{\mathcal{C}}(s)}+\mathcal{J}^{\pi_k}_{h-1}(s^\prime)).$ Thus, $\mathcal{J}^{\pi_k}_H(s)$ represents the amount of constraint violations within horizon $H$ when following the current policy $\pi_k$. We can then define the complete safe set $S^{\pi_k}_{\mathcal{C}}(s)$ for state $s$ under policy $\pi_k$ at iteration $k$ by $S^{\pi_k}_{\mathcal{C}}(s)=\left\{a\middle|\mathcal{J}^{\pi_k}_H(s|a)=0\right\}$. At policy improvement, the policy is updated by: $$\pi_{k+1}(s)\leftarrow \argmax\limits_{a\in S^{\pi_k}_{\mathcal{C}}(s)}\sum\limits_{s^\prime}p(s^\prime|s,a)\left(r(s,a)+\gamma V_{\mathcal{C}}^{\pi_k}(s^\prime)\right).$$ Therefore, by definition $\pi_{k+1}(s)\in S^{\pi_k}_{\mathcal{C}}(s).$ The monotonic improvement of Policy Iteration (PI) holds for CPI w.r.t. the constrained value-function $V_\mathcal{C}^\pi(s)$:
\begin{align*}
    V_\mathcal{C}^{\pi_k}(s)&\leq \max_{a\in S^{\pi_k}_{\mathcal{C}}(s)} Q_\mathcal{C}^{\pi_k}(s,a)\\
    &=\max_{a\in S^{\pi_k}_{\mathcal{C}}(s)}r(s,a)+\gamma\sum\limits_{s^\prime}p(s^\prime|s,a)V_\mathcal{C}^{\pi_k}(s^\prime)\\
    &=r(s,\pi_{k+1}(s))+\gamma\sum\limits_{s^\prime}p(s^\prime|s,\pi_{k+1}(s))V_\mathcal{C}^{\pi_k}(s^\prime)\\
    &\leq r(s,\pi_{k+1}(s))+\\
    &\qquad\gamma\sum\limits_{s^\prime}p(s^\prime|s,\pi_{k+1}(s))\max_{a\in S^{\pi_k}_{\mathcal{C}}(s)}Q_\mathcal{C}^{\pi_k}(s^\prime,a)\\
    &=V_\mathcal{C}^{\pi_{k+1}}(s)
\end{align*}
Optimality then follows from the optimality of PI \cite{DBLP:books/wi/Puterman94}.
\end{proof}
In the following, we focus on the empirical evaluation of CQL in the tabular and function approximation settings and leave further theoretical analysis as future work.

\section{Deep Constrained Q-learning}
\label{sec:dltcq}

In order to employ Constrained Q-learning within DQN, the target has to be modified to $y^Q_i=r_i + \gamma \mathop{\max_{\ a \in S_{\mathcal{C}}(s_{i+1}) }} Q'(s_{i+1}, a|\theta^{Q'})$,
where $Q'$ is the target-network, parameterized by $\theta^{Q'}$\hspace*{-0.075cm}. Parameters $\theta^Q$ are then updated on the mean squared error and the parameters of the target network by Polyak averaging. We refer to this algorithm as \textit{Constrained DQN} (CDQN). A general description is shown in \Cref{alg:ltdqn}.

\begin{algorithm}[h]
\small
    \SetAlgoLined
    \DontPrintSemicolon
    initialize $Q$ and $Q'$ and set replay buffer $\mathcal{R}$\\
    \For{\text{optimization step} o=1,2,\dots}{
        sample minibatch $(s_i, a_i, s_{i+1}, r_i)_{1\leq i\leq n}$ from $\mathcal{R}$\\
        
        calculate safe sets $S_{\mathcal{C}}(s_{i+1})_{1\leq i\leq n}$ for all constraints\\
        set $y^Q_i=r_i + \gamma \mathop{\max_{\ a \in S_{\mathcal{C}}(s_{i+1}) }} Q'(s_{i+1}, a|\theta^{Q'})$\\
        minimize MSE of $y^Q_i$ and $Q(s_i, a_i|\theta^Q)$\\
        update target network $Q'$\\
    }
    \For{\text{execution step} e=1,2,\dots}{
        get current state $s_t$ from environment\\
        calculate safe set $S_{\mathcal{C}}(s_{t})$ for all constraints\\
        apply $\pi(s_t) = \argmax_{a\in S_\mathcal{C}(s_t)}Q(s_t, a)$ 
      }
    \caption{(Fixed Batch) CDQN}
    \label{alg:ltdqn}
\end{algorithm}

\section{Deep Multi-step Constraints}
\label{sec:dmsc}
To cope with infinite state-spaces, we jointly estimate $\mathcal{J}^{\pi}_h|_{1\leq h\leq H}$ with function approximator $\mathcal{J}(\cdot, \cdot|\theta^\mathcal{J})$, parameterized by $\theta^\mathcal{J}$. Based on consecutive bootstrapping, the targets are given by $y^\mathcal{J}_{t,1}=j_t$ and $y^\mathcal{J}_{t,h>1}=j_t + \mathcal{J}^\prime_{h-1}(s_{t+1}, \argmax_{a \in S_{\mathcal{C}}(s_{t+1})} Q(s_{t+1}, a|\theta^Q)|\theta^{\mathcal{J}^\prime_{h-1}})$,
 where $\mathcal{J}^\prime_\cdot$ represent target networks with parameters $\theta^{\mathcal{J}^\prime_{\cdot}}$ (cf. \Cref{alg:ltdqnmsc}). We update $\theta^{\mathcal{J}^\prime_{\cdot}}$ by Polyak averaging.
 \begin{algorithm}[h]
\small
    \SetAlgoLined
    \DontPrintSemicolon
    initialize $Q$ and $Q'$ and set replay buffer $\mathcal{R}$\\
    initialize all multi-step constraints $\mathcal{J}$ and $\mathcal{J}'$\\ 
    \For{\text{optimization step} o=1,2,\dots}{
        sample minibatch $(s_i, a_i, s_{i+1}, r_i)_{1\leq i\leq n}$ from $\mathcal{R}$\\
        
        calculate safe sets $S_{\mathcal{C}}(s_{i+1})_{1\leq i\leq n}$ for all constraints\\
        set $y^Q_i=r_i + \gamma \mathop{\max_{\ a \in S_{\mathcal{C}}(s_{i+1}) }} Q'(s_{i+1}, a|\theta^{Q'})$\\
        minimize MSE of $y^Q_i$ and $Q(s_i, a_i|\theta^Q)$\\
        update target network $Q'$\\
        
        \ForEach{\text{multi-step constraint} $\mathcal{J}$}{
            set multi-step constraint targets $y^{\mathcal{J}}_{i,h\leq H}$\\
            minimize MSE of $y^{\mathcal{J}}_{i,h\leq H}$ and $\mathcal{J}_{h\leq H}(s_i, a_i|\theta^{\mathcal{J}})$ \\
            update target networks ${\mathcal{J}}'$\\
        }
    }
    \For{\text{execution step} e=1,2,\dots}{
        get current state $s_t$ from environment\\
        calculate safe set $S_{\mathcal{C}}(s_{t})$ for all constraints\\
        apply $\pi(s_t) = \argmax_{a\in S_\mathcal{C}(s_t)}Q(s_t, a)$ 
      }
    \caption{CDQN with Multi-step Constraints}
    \label{alg:ltdqnmsc}
\end{algorithm}

\section{Experiments}
In order to analyze the effect of Constrained Q-learning, we first evaluate Constrained Q-learning in the tabular setting, before we compare CDQN to reward shaping and loss-penalties for the more complex task of high-level decision making for autonomous lane changes.

\section{Tree MDPs}
We apply tabular Constrained Q-learning to the MDP-class described in \Cref{fig:tabmdpclass}, an extension of the motivational example in \Cref{fig:motivation} with an arbitrary number of $B$ distracting paths, subsequently called branches. All transitions except for terminal ones have an immediate reward of 0. The return of each path is indicated at the right. All upper paths yielding a higher return than the safe path at the bottom lead through some unsafe state which has to be avoided (e.g. by safety requirements), except for the second lowest path which on the other hand yields the lowest return. Thus, it is optimal to pick the lowest path at the decision step in state $s_1$.\\
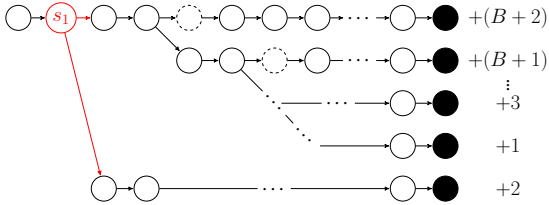
\begin{figure}[h]
  \centering  
  \resizebox{0.41\textwidth}{!}{%
    \begin{tikzpicture}[->,>=stealth',shorten >=1pt,auto, node distance=1.5cm, semithick]
        \tikzstyle{every state}=[]
        
        \node[state] (0) {};
        \node[state] (1) [right of=0, red] {\Huge$s_1$};
        \node[state] (2) [right of=1] {};
        \node[state] (3) [right of=2] {};
        \node[state] (4) [right of=3, dashed] {};
        \node[state] (5) [right of=4] {};
        \node[state] (6) [right of=5] {};
        \node[state] (7) [right of=6] {};
        \node[] (8) [right of=7] {\Huge\dots};
        \node[state] (9) [right of=8] {};
        \node[state] (10) [right of=9, fill=black] {};
        \node[state] (11) [below of=4] {};
        \node[state] (12) [right of=11] {};
        \node[state] (13) [right of=12, dashed] {};
        \node[state] (14) [right of=13] {};
        \node (15) [right of=14] {\Huge\dots};
        \node[state] (16) [right of=15] {};
        \node[state] (17) [right of=16,fill=black] {};
        \node (18) [below of=13, rotate=315] {\Huge\dots};
        
        \node[state] (19) [below of=17, fill=black] {};
        \node[state] (20) [left of=19] {};
        \node (21) at ($(18)!0.5!(20)$) {\Huge\dots};
        
        \node[state] (29) [below of=20] {};
        \node[state] (30) [below of=19, fill=black] {};
        
        \node (31) at (18|-30) [rotate=315, yshift=1.0cm, xshift=0.5cm] {\Huge\dots};
        
        \node (22) at (18|-30) [rotate=315, yshift=1.0cm, xshift=1cm] {};
        \node (23) at (2|-22) {};
        \node[state] (24) [below of=23] {};
        \node[state] (25) [right of=24] {};
        \node[state] (26) at (25-|19) [fill=black] {};
        \node[state] (27) [left of=26] {};
        \node (28) at ($(25)!0.5!(27)$) {\Huge\dots};
        
        \node (32) [right of=10, xshift=0.7cm] {\Huge $+(B+2)$};
        \node (33) [right of=17, xshift=0.7cm] {\Huge $+(B+1)$};
        \node (34) [right of=19, xshift=0.7cm] {\Huge $+3$};
        \node (35) [right of=30, xshift=0.7cm] {\Huge $+1$};
        \node (36) [right of=26, xshift=0.7cm] {\Huge $+2$};
        \node (37) at ($(33)!0.5!(34)$) {\Huge $\vdots$};
        
        \path (0) edge [] node {} (1)
              (1) edge [red] node {} (2)
              (2) edge [] node {} (3)
              (3) edge [] node {} (4)
              (4) edge [] node {} (5)
              (5) edge [] node {} (6)
              (6) edge [] node {} (7)
              (7) edge [] node {} (8)
              (8) edge [] node {} (9)
              (9) edge [] node {} (10)
              (3) edge [] node {} (11)
              (11) edge [] node {} (12)
              (12) edge [] node {} (13)
              (13) edge [] node {} (14)
              (14) edge [-] node {} (15)
              (15) edge [] node {} (16)
              (16) edge [] node {} (17)
              (12) edge [-] node {} (18)
              (18) edge [-] node {} (21)
              (21) edge [] node {} (20)
              (20) edge [] node {} (19)
              (1) edge [red] node {} (24)
              (24) edge [] node {} (25)
              (25) edge [-] node {} (28)
              (28) edge [] node {} (27)
              (27) edge [] node {} (26)
              (18) edge [-] node {} (31)
              (22) edge [] node {} (29)
              (29) edge [] node {} (30);
              
    \end{tikzpicture}
}

\long\def\/*#1*/{}

\/*
\resizebox{0.5\textwidth}{!}{%
\begin{tikzpicture}[->,>=stealth',shorten >=1pt,auto,node distance=2.8cm,
                    semithick]
  \tikzstyle{every state}=[]

  \node[state] (0) {};
  
  \node[state] (1) [right=0.5cm of 0] {};
  \node       (d1) [above=0.5cm of 1] {};
  \node       (d2) [below=0.5cm of 1] {};
  
  \node       (d3) [right=0.5cm of 1] {};
  \node[state] (2) [right=0.5cm of d1] {};
  \node[state] (3) [right=0.5cm of d2] {};

  \node       (d4) [right=0.5cm of d3] {};
  \node[state] (4) [right=0.5cm of 2] {};
  \node[state] (5) [right=0.5cm of 3] {};

  \node[state] (7) [right=0.5cm of d4] {};
  \node[state] (6) [right=0.5cm of 4, dashed] {};
  \node[state] (8) [right=0.5cm of 5] {};

  \node[state] (9) [right=0.5cm of 6] {};
  \node[state] (10) [right=0.5cm of 7] {};
  \node[state] (11) [right=0.5cm of 8] {};

  \path (0) edge [below, very thick] node {} (1)
        (1) edge [very thick, sloped, anchor=center, below] node {} (2)
        (1) edge [very thick, sloped, anchor=center, below] node {} (3)
        (2) edge [very thick, below] node {} (4)
        (4) edge [very thick] node {} (6)
        (6) edge [very thick] node {} (9)
        (4) edge [very thick, sloped, anchor=center, below] node {} (7)
        (7) edge [very thick, below] node {} (10)
        (3) edge [very thick, below] node {} (5)
        (5) edge [very thick, below] node {} (8)
        (8) edge [very thick, below] node {} (11);
\end{tikzpicture}
}*/
  \caption{MDP class considered in the tabular experiments. Dashed states are marked as unsafe, terminal states are black. All transitions have an immediate reward of $+0$, terminal rewards are given to the right. The optimal path is the one at the bottom with a return of $+2$, as the only safe path above yields a return of $+1$.}
  \label{fig:tabmdpclass}
\end{figure}

Without reward shaping, vanilla Q-learning would pick the path at the top, since it yields the highest return. The strongest learning signal which can be included to augment the reward function is to give an immediate reward of $-\infty$ if an unsafe state is reached in a given transition (in the function approximation setting, this can be infeasible due to stability issues caused by the variance increase in the reward signal). Hence, we compare Constrained Q-learning to shaped Q-learning with this augmented reward. Results can be seen in \Cref{fig:tabresults}. Compared to vanilla Q-learning, Constrained Q-learning needs $25\%$ less samples until convergence to the optimal policy in case of 1 branch and up to $90\%$ for 10 branches.\\
\begin{figure}[h]
    \centering
    \includegraphics[width=0.44\textwidth]{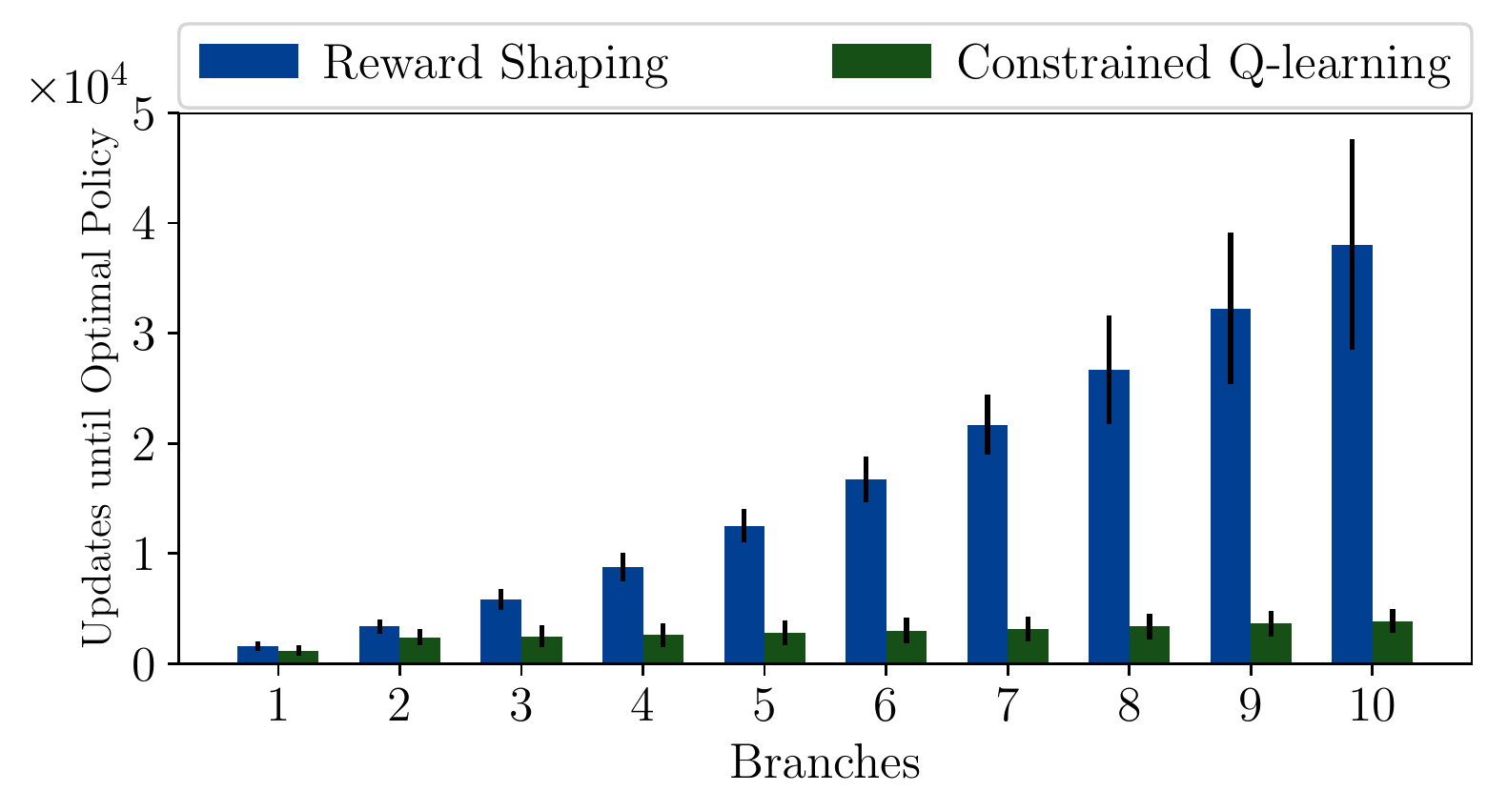} 
    \caption{Results of CQL and RS for the MDP class shown in \Cref{fig:tabmdpclass} for $1$ to $10$ branches.}
    \label{fig:tabresults}
\end{figure}

These results demonstrate the effect of action-space shaping, i.e. the underlying mechanism of Constrained Q-learning, compared to reward shaping. The incorporation of prior knowledge at maximization in the Q-update restricts the search space over Q-functions which simplifies the learning problem. In the following, we further underline this finding in the function approximation setting.

\section{High-Level Decision Making for Autonomous Driving}
We evaluate Constrained DQN on the task of learning autonomous lane changes in the open-source simulator SUMO. We build upon the MDP formulation in \cite{DeepSetQ} and further use the same settings for SUMO. However, we change the value of \textit{lcKeepRight} to be in $\{5,8,10\}$ for the meta-configurations of the driver types, in order to enforce the drivers to keep right. For all experiments, we use the network architecture from  DeepSet-Q \cite{DeepSetQ} to deal with a variable number of surrounding vehicles. Since it is infeasible to train such a deep RL system online due to the high demand for executed actions, we train all methods on the same fixed batch of $5\cdot10^5$ transitions collected by an exploratory policy with random sampling for $2.5\cdot10^6$ gradient steps. We then further employ CDQN on the real HighD data set \cite{highDdataset}, which instead contains transitions collected by real human drivers. The discrete action space $\mathcal{A}$ includes three actions: \emph{keep lane}, \emph{perform left lane change} and \emph{perform right lane change}. Acceleration and maintaining safe-distance to the preceding vehicle on the same lane are controlled by a low-level SUMO controller. The primary objective is to drive as close as possible to a desired velocity. Thus, we define the reward function $r: \mathcal{S}\times\mathcal{A}\mapsto \mathbb{R}$ as $r(s, a) = r_\text{speed}(s, a) = 1 - \frac{|v_{\text{current}}(s)- v_{\text{desired}}(s)|}{v_{\text{desired}}(s)},$ where $v_{\text{current}}$ and $v_{\text{desired}}$ are the actual and desired velocity of the agent. In this work, we focus on the combination of two single-step and one multi-step constraint:
\paragraph{Safety} To guarantee safe lane changes, we use a constant velocity planner to predict future outcomes based on the \textit{Intelligent Driver Model} \cite{treiber2000congested} which ensures a desired time headway to preceding vehicles. We formulate the constraint signal as $c_\text{safe}(s,a) = \mathds{1}_{a \text{ is not safe}}$. Additionally, we restrict lane changes on the outermost lanes (it is not allowed to drive outside the lanes) by using a second constraint signal $c_{\text{lane}}(s, a) = \mathds{1}_{l_\text{next} < 0}+\mathds{1}_{l_\text{next} \geq \text{num lanes}}$. The safe set of the safety constraint can then be formulated as $S_{\text{safety}}(s, a) = \{ a \in \mathcal{A} | c_\text{safe}(s,a) + c_\text{lane}(s,a) \le 0 \}$. Driving straight is always safe. In case of contradicting constraints we give safety higher priority.

\paragraph{Keep-Right}
As second constraint, we add a keep-right rule. The agent ought to drive right when there is a gap of at least $t_\text{gap}$ (we set $t_\text{gap}$ to $\SI{10}{s}$ in our experiments) on the same lane and on the lane right to the agent assuming driving with the desired velocity before the closest leader is reached. A corresponding rule is part of the traffic regulations in Germany. We can then formulate the constraint signal as $c_{\text{r}}(s, a) = \mathds{1}_{a\neq\text{right and }\Delta t_\text{right} > t_\text{gap} \text{ and }  \Delta t_\text{same} > t_\text{gap}}$, where $\Delta t$ is the true gap time span. Additionally, the agent is not allowed to leave its current lane, if there is no leader on the same lane or one lane to the left, i.e. $c_\text{l}(s, a) = \mathds{1}_{a=\text{left and }\Delta t_\text{left} > t_\text{gap} \text{ and }\Delta t_\text{same} > t_\text{gap}}$. The safe set thus becomes $S_{\text{KR}}(s, a) = \{ a \in \mathcal{A} | c_\text{r}(s,a) + c_\text{l}(s,a) \le 0 \}$, where \textit{KR} abbreviates \textit{Keep-Right}.

\paragraph{Comfort}
In order to guarantee comfort, we approximate a multi-step prediction of lane changes based on our target-policy in a model-free manner. We set the immediate lane change value $j_t$ to 1, if the agent performs a lane change and 0 otherwise. Within the defined time span, a maximum of $\beta_{\text{LCmax}}$ lane changes are allowed. We calculate the amount of lane changes over $H=5$ (10s) by using $\mathcal{J}_{5}^\pi$, i.e. the safe set can be defined by $S_{\text{LCmax}}(s, a) = \{ a \in \mathcal{A} | \mathcal{J}_{5}^\pi(s,a) \le \beta_{\text{LCmax}} \}$. In our experiments, we set $\beta_{\text{LCmax}}=2$. Lowering the threshold for a fixed horizon results in a more conservative behaviour, since less lane changes are allowed. Increasing the threshold adds flexibility to the behaviour of the agent, however, it will most probably lead to more lane changes. The same holds for a fixed threshold and varying horizon analogously. A hard constraint on the number of lane-changes can be avoided by an alternative formulation of the comfort multi-step constraint, where  optional lane-changes are performed only if the expected velocity increases by a certain amount in a certain time. We define the immediate gain as $j_t = v_{t+1} - v_{t}$ and the corresponding safe set as $S_{\text{VGmin}}(s, a) = \{ a \in \mathcal{A} | \mathcal{J}_{5}^\pi(s,a) \ge \beta_{\text{VGmin}}\}$. We only allow additional lane-changes, if the velocity gain over $H=5$ exceeds $\beta_{\text{VGmin}}=\SI{0.25}{\meter\per\second}$. We estimate multi-step constraint-values and Q-values in one architecture (using multiple output heads in the last layer) to speed up training. The optimized architecture is shown in \Cref{fig:architecture}.

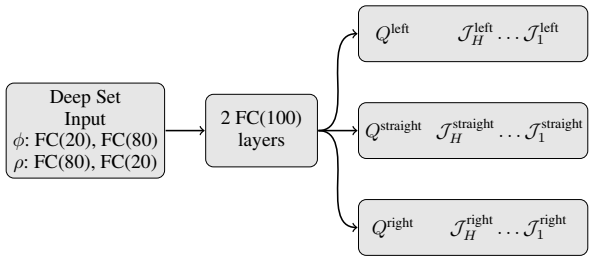
\begin{figure}[h]
\centering
  \resizebox{0.44\textwidth}{!}{%
\begin{tikzpicture}
\tikzset{line width=1pt}

\tikzstyle{box_dashed}=[rectangle, draw, rounded corners, dashed]
\tikzstyle{box_dotted}=[rectangle, draw, rounded corners, dotted]
\tikzstyle{element}=[draw, fill=black!10, text centered] 
\tikzstyle{element_rect}=[element, rounded corners, rectangle, minimum height = 0.8cm, minimum width = 1.5cm]
\tikzstyle{element_ell}=[element, ellipse]
\tikzstyle{element_circ}=[element, circle]
\tikzstyle{element_dia}=[element, diamond]
\tikzstyle{element_rect_split}=[element, rounded corners, rectangle split, rectangle split horizontal, rectangle split parts=2, rectangle split draw splits=false]

\node (layer1) [element_rect, minimum width = 2cm,align=center, minimum height = 1.3cm] {Deep Set\\Input\\$\phi$: FC(20), FC(80)\\$\rho$: FC(80), FC(20)};
\node (layer3) [element_rect, minimum width = 2cm, right=0.7cm of layer1,align=center, minimum height = 1.3cm] {2 FC(100)\\layers};
\node (layer5) [element_rect_split, right=0.7cm of layer3, minimum height = 1cm] {\nodepart[text width=1cm]{one}$Q^\text{straight}$\nodepart[text width=2.6cm]{two}$\mathcal{J}^\text{straight}_H \dots \mathcal{J}^\text{straight}_{1}$};
\node (layer4) [element_rect_split, above=0.7cm of layer5, minimum height = 1cm] {\nodepart[text width=1cm]{one}$Q^\text{left}$\nodepart[text width=2.6cm]{two}$\mathcal{J}^\text{left}_H \dots \mathcal{J}^\text{left}_{1}$};
\node (layer6) [element_rect_split, below=0.7cm of layer5, minimum height = 1cm] {\nodepart[text width=1cm]{one}$Q^\text{right}$\nodepart[text width=2.6cm]{two}$\mathcal{J}^\text{right}_H \dots \mathcal{J}^\text{right}_{1}$};

\draw [] (layer1.east) edge[->, thick, out=0, in=180] (layer3.west);
\draw [] (layer3.east) edge[->, thick, out=0, in=180] (layer4.west);
\draw [] (layer3.east) edge[->, thick, out=0, in=180] (layer5.west);
\draw [] (layer3.east) edge[->, thick, out=0, in=180] (layer6.west);



\end{tikzpicture}
}
  \caption{Architecture of Constrained DQN analogous to \cite{DeepSetQ} with modified output to jointly estimate Q- and comfort-values.}
  \label{fig:architecture}
\end{figure}

 More constraints can be easily added in the same manner. To highlight the advantages of CDQN, we compare to the following baselines:
\paragraph{Safe Policy Extraction (SPE)} In this baseline, we check for constraints only at policy extraction.

\paragraph{Reward Shaping} We compare to a reward shaping approach, where we add weighted penalties for lane changes and for not driving on the right lane, i.e.:
$r(s, a) = r_\text{speed}(s,a) - \lambda_\text{LC}p_\text{LC} - \lambda_\text{KR}p_\text{KR}$. We set $p_{\text{LC}} = 1$ if action $a$ is a lane change and $0$ otherwise. Further, we set $p_\text{KR} = l_\text{current}$ for the current lane index  $l_\text{current}$, where lane index 0 is the right most lane.

\paragraph{Additional Loss Terms}
As an alternative, we  approximate the solution of our constrained MDP using the reward $r_\text{speed}(s,a)$ by  including constraint penalties in the loss. We penalize the objective for constraint violations, solving the constrained surrogate:  
\vspace*{-0.1cm}
\begin{align*}
L(\theta^Q) =& \frac{1}{n}\sum\limits_{i=1}^n (y_i - Q(s_i, a_i|\theta^Q))^2  + ( \lambda_\text{safe}\mathds{1}_{a_i\not\in S_{\text{safe}}} \\
& +\lambda_\text{KR} \mathds{1}_{a_i\not\in S_{\text{KR}}}
+\lambda_\text{comfort} \mathds{1}_{a_i\not\in S_{\text{comfort}}} ) Q(s_i, a_i | \theta^Q)^2
\end{align*}

We multiply the constraint masks by the squared Q-values to penalize constrained violations according to their value.\\

\begin{figure}[t]
\vspace*{0.05cm}
    \centering
   \includegraphics[width=0.43\textwidth]{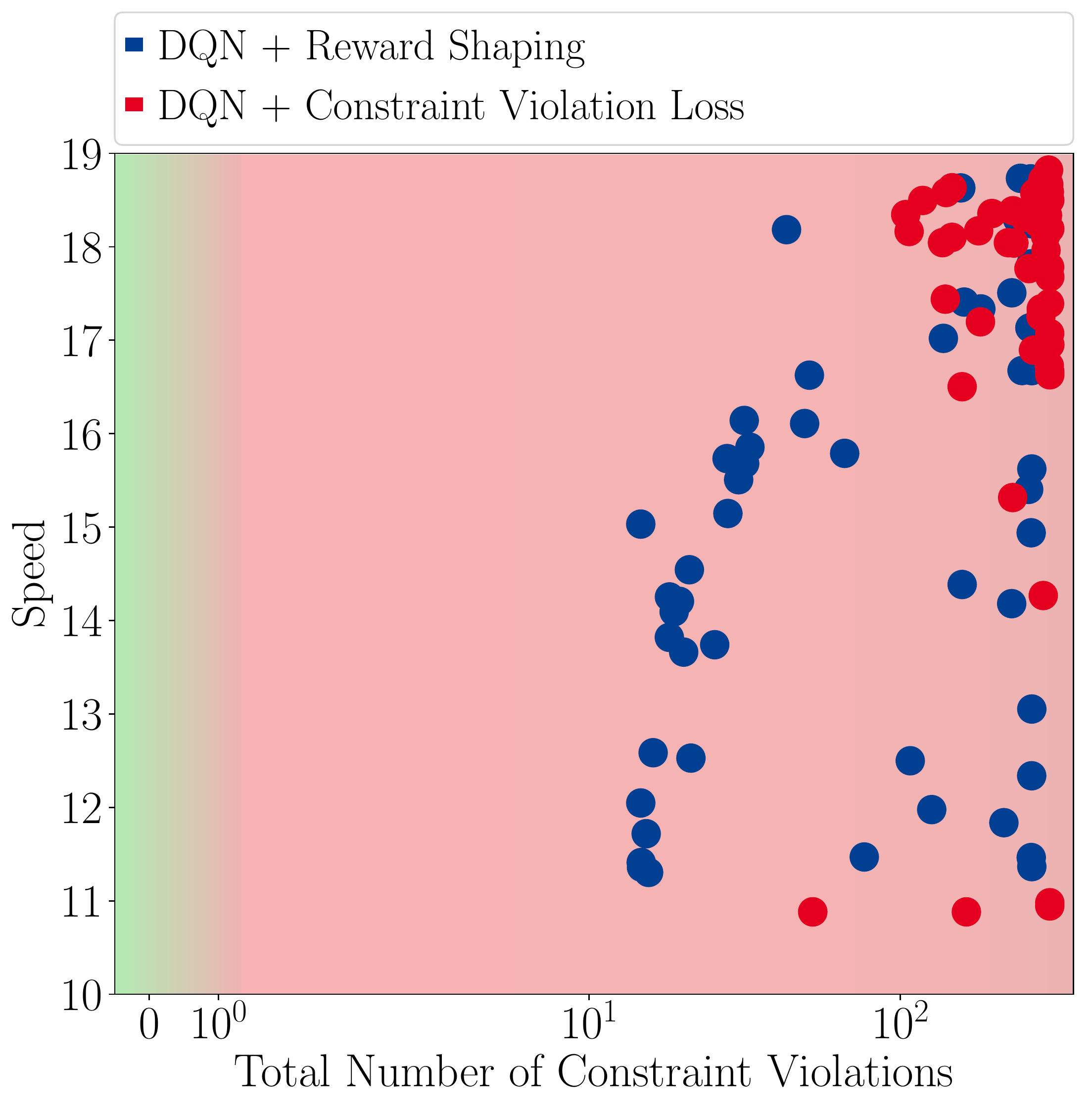}
 
    \caption{Mean performance of 10 training runs for scenarios with 20 to 80 vehicles. The average speed is shown on the y-axis and average number of comfort and KR constraint violations on the x-axis. Every point corresponds to one of 50 sampled configurations by random search for different (blue) reward shaping weights and (red) loss penalty weights.}
    \label{fig:revio}
\end{figure}
\begin{figure}[t]
    \centering
    \includegraphics[width=0.43\textwidth]{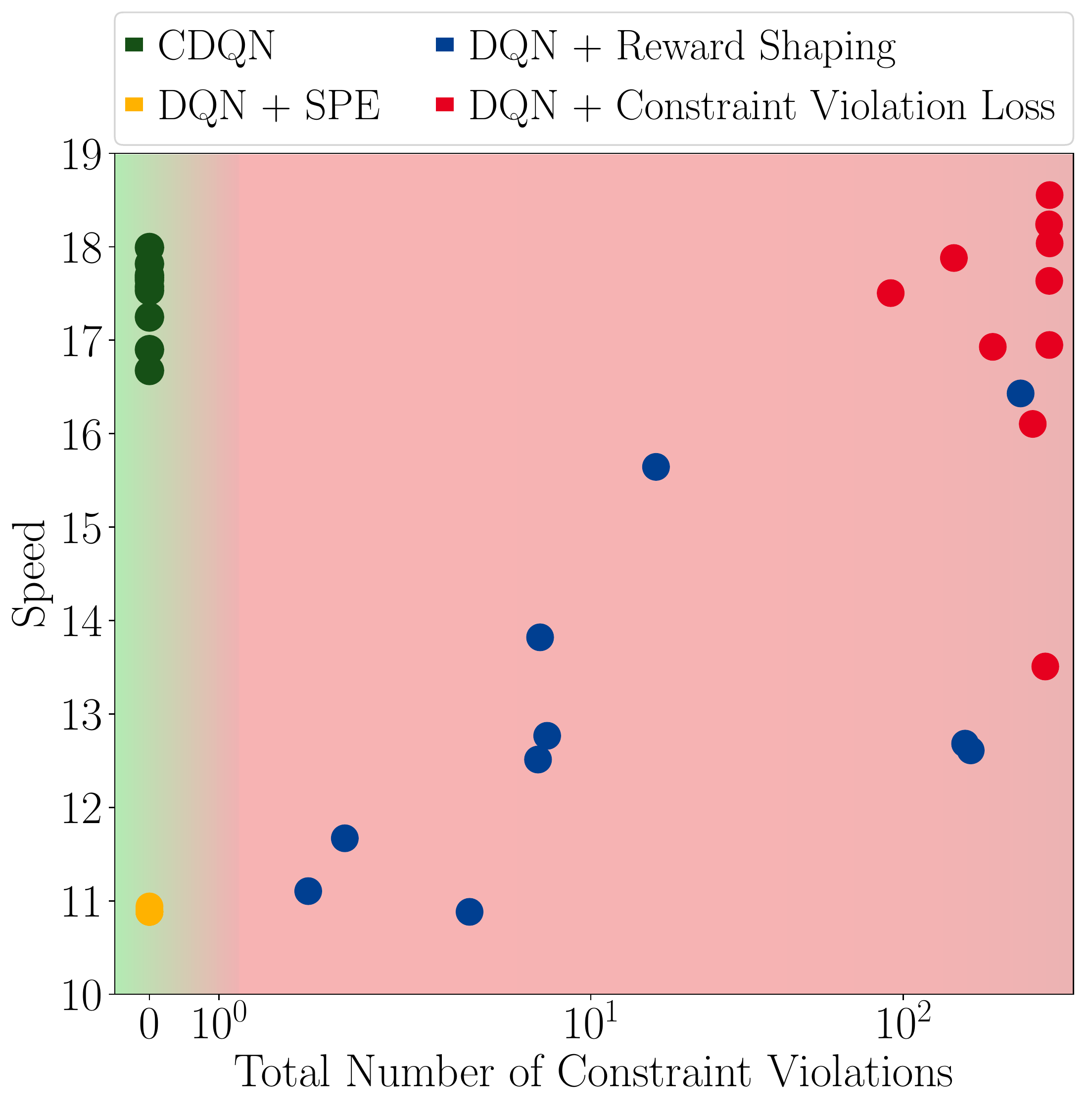}
    \caption{Mean performance of 10 training runs for scenarios with 20 to 80 vehicles. The average speed is shown on the y-axis and average number of comfort and KR constraint violations on the x-axis. \MTSCDQN  is compared to DQN with reward shaping, constraint violation loss and SPE.}
    \label{fig:longterm}
\end{figure}

The penalty-weights for the baselines were optimized with random search using a fixed budget of $1.25\cdot10^6$ gradient steps due to computational costs. In total, we sampled 50 configurations for each baseline. The results of the random search are shown in \Cref{fig:revio}, which indicates the total amount of comfort and \textit{KR} constraint violations and the speed for each configuration. The safety constraint was enabled for policy extraction in both methods, hence there were no safety constraint violations. Please note, however, that only the action-value in CQL represents a safe policy optimal w.r.t. the \emph{long-term} return. As incumbent, we chose the configuration with the lowest number of violations for which the policy did not collapse to only driving straight. For both methods, configurations show either high speed in combination with a high number of constraint violations, or they violate a low number of constraints but are quite slow. This underlines the difficulty of finding proper settings in both reward shaping and Lagrangian methods.
\subsection{Real Data}
In order to evaluate the real-world applicability of our approach, we generate a transition set from the open HighD data set \cite{highDdataset}, containing 147 hours of top-down recordings of German highways. The data set includes features for the different vehicles, such as a distinct ID, velocity, lane and position. We discretize \SI{5}{\second} before and after occurrences of lane changes with a step size of \SI{2}{\second}, leading to a consecutive chain of 5 time steps with one lane change per chain. The acting vehicle is then considered as the current agent. In total, this results in a replay buffer of $\sim20000$ transitions with $\sim5000$ lane changes.

\section{Results}

The results for agents considering the three defined constraints can be found in \Cref{fig:longterm}. \MTSCDQN is the only agent satisfying all constraints in every time step while taking most advantage of the maximum allowed number of lane changes, showing high speed and the lowest variance. All other agents are not able to drive close to the desired velocity or cause a tremendous amount of constraint violations. The results of reward shaping and Lagrangian optimization suffer from high variance and are not consistent. The worst performance is shown by the SPE agent,
staying on the initial lane with no applied lane changes over all training runs. Thus, \MTSCDQN is by far the best performing agent, driving comfortable and fast without any violations. Both comfort constraint formulations led to an equivalent behavior (therefore, additional results not shown). A comparison of DQN with SPE and \MTSCDQN trained in simulation to \MTSCDQN trained on the open HighD data set \cite{highDdataset} is shown in \Cref{fig:highdres}. While there is a larger difference in performance for simulated scenarios with 50 and more vehicles, the agents trained in simulation and the real data perform equivalently for scenarios with 20 to 40 vehicles. Furthermore, \MTSCDQN trained on real data outperforms DQN with SPE trained directly in simulation, which is not capable to solve the task adequately while satisfying all constraints in all time steps.  The learned policy of \MTSCDQN  generalizes to new scenarios and settings, even with mismatches between simulation and the real recordings.

\begin{figure}[t]
    \centering
   \includegraphics[width=0.43\textwidth]{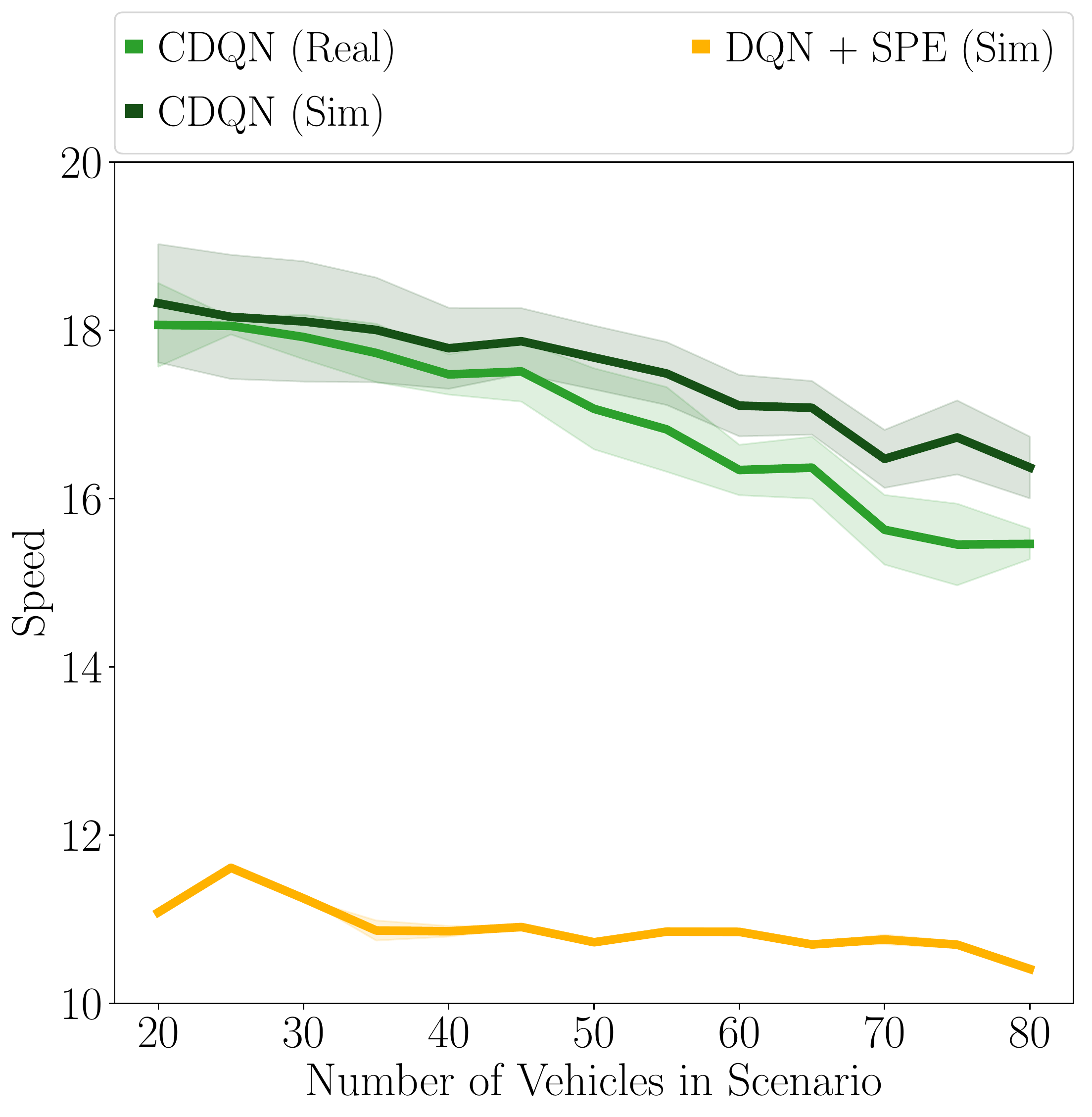}
 
    \caption{Mean performance and standard deviation of CDQN trained in simulation (Sim) and the real HighD data set \cite{highDdataset} (Real). Evaluated in simulation for scenarios with 20 to 80 vehicles. Further comparison to DQN with Safe Policy Extraction trained and evaluated in simulation.}
    \label{fig:highdres}
\end{figure}

\section{Conclusion}
We introduced Constrained Q-learning, an approach to incorporate hard constraints directly in the Q-update to find the optimal deterministic policy for the induced constrained MDP. For its formulation, we define a new class of multi-step constraints based on truncated value-functions. In the tabular setting, Constrained Q-learning proved to be 10 times more sample-efficient than reward shaping which underlines the finding that action-space shaping imposes an easier learning problem compared to reward shaping. In high-level decision making for autonomous driving, \MTSCDQN is outperforming reward shaping, Lagrangian optimization and Safe Policy Extraction in terms of final performance with orders of magnitude less constraint violations, while offering more interpretable constraint formulations by avoiding the need for discounting. \MTSCDQN can learn an optimal safe policy directly from real transitions without the need of simulated environments, which is a major step towards the application to real systems.

\bibliography{constrained_q}

\end{document}